\documentclass[journal]{IEEEtran}
\newcount\DraftStatus  
\usepackage{comment}
\excludecomment{JournalOnly}  
\includecomment{ConferenceOnly}  
\excludecomment{TulipStyle}
\newcount\DraftStatus  
\ifnum\DraftStatus=1
	\usepackage[draft,colorinlistoftodos,color=orange!30]{todonotes}
\else
	\usepackage[disable,colorinlistoftodos,color=blue!30]{todonotes}
\fi 
\usepackage[draft]{todonotes}   

\makeatletter
 \providecommand\@dotsep{5}
 \def\listtodoname{List of Todos}
 \def\listoftodos{\@starttoc{tdo}\listtodoname}
 \makeatother
\usepackage{color}
\definecolor{officegreen}{rgb}{0.0, 0.5, 0.0}

\usepackage{graphicx}
\graphicspath{{./figures/}{./graphics/}{./graphics/logos/}}

\newtheorem{assumption}{Assumption}

\usepackage{graphicx}
\usepackage{algorithm}
\usepackage[noend]{algpseudocode}

\usepackage{breqn}
\usepackage{subcaption}
\usepackage{multirow}
\usepackage{psfrag}
\usepackage{url}
\usepackage{booktabs}
\usepackage{rotating}
\usepackage{colortbl}
\usepackage{paralist}
\usepackage{epstopdf}
\usepackage{nag}
\usepackage{microtype}
\usepackage{siunitx}
\usepackage{nicefrac}
\usepackage{lipsum}
\usepackage[english]{babel}
\usepackage[pangram]{blindtext}
\usepackage{tikz}
\usepackage{tabularx}
\usetikzlibrary{fit,positioning,arrows.meta,shapes,arrows}

\usepackage[numbers]{natbib}

\usepackage{xcolor}
\usepackage{pdfcomment}

\newtheorem{theorem}{Theorem}

\usepackage{stackengine}
\setstackgap{L}{.5\baselineskip}

\usepackage{textcomp}

\usepackage{amsfonts}

\usepackage{braket}

\usepackage{mathptmx}

\usepackage{notoccite}

\IEEEoverridecommandlockouts
\def\BibTeX{{\rm B\kern-.05em{\sc i\kern-.025em b}\kern-.08em
    T\kern-.1667em\lower.7ex\hbox{E}\kern-.125emX}}

\begin{document}
%
%
\title{Decentralized Quantum Federated
Learning for Metaverse: Analysis, Design and Implementation
}


\author{Dev Gurung, Shiva Raj Pokhrel and Gang Li 

\thanks{The authors are with the School of Information Technology, Deakin University, Geelong, VIC 3220, Australia}}
\maketitle

\begin{abstract}
With the emerging developments of the Metaverse,  a virtual world where people can interact, socialize, play, and conduct their business, it has become critical to ensure that the underlying systems are transparent, secure, and trustworthy. 
To this end, we develop a decentralized and trustworthy quantum federated learning (QFL) framework. 
The proposed QFL leverages the power of blockchain to create a secure and transparent 
system that is robust against cyberattacks and fraud. 
In addition, the decentralized QFL system addresses the risks associated with a centralized server-based approach.
With extensive experiments and analysis, we evaluate classical federated learning (CFL) and QFL in a distributed setting and demonstrate the practicality and benefits of the proposed design. 
Our theoretical analysis and discussions develop a genuinely decentralized financial system essential for the Metaverse.
Furthermore, we present the application of blockchain-based QFL in a hybrid metaverse powered by a metaverse observer and world model.
Our \href{https://github.com/s222416822/BQFL} {implementation details and code} are publicly available \footnote{https://github.com/s222416822/BQFL}.
\end{abstract}

\begin{IEEEkeywords}
Quantum Federated Learning, Metaverse, Blockchain
\end{IEEEkeywords}

\section{Introduction}\label{sec-intro}
Quantum Machine Learning (QML)~\cite{biamonte2017quantum} has emerged as a promising paradigm in a number of computationally demanding fields, thanks to the proliferation of quantum computers and the ensuing surge in linear/algebraic computation and operational capabilities.
The underlying physics of QML, such as entanglements, teleportation, and superposition \cite{kwakQuantumDistributedDeep2022} are the key enablers of such high computational efficiency. 
When such enablers can be developed under the  Federated Learning (FL) framework, such as by employing Quantum Neural Networks 
(QNNs) \cite{abbasPowerQuantumNeural2021}, the training, 
learning, federation, prediction, and optimization capabilities can leapfrog simultaneously, leading to the development of Quantum FL (QFL).

\begin{figure}
    \centering
    \includegraphics[width=0.75\columnwidth]{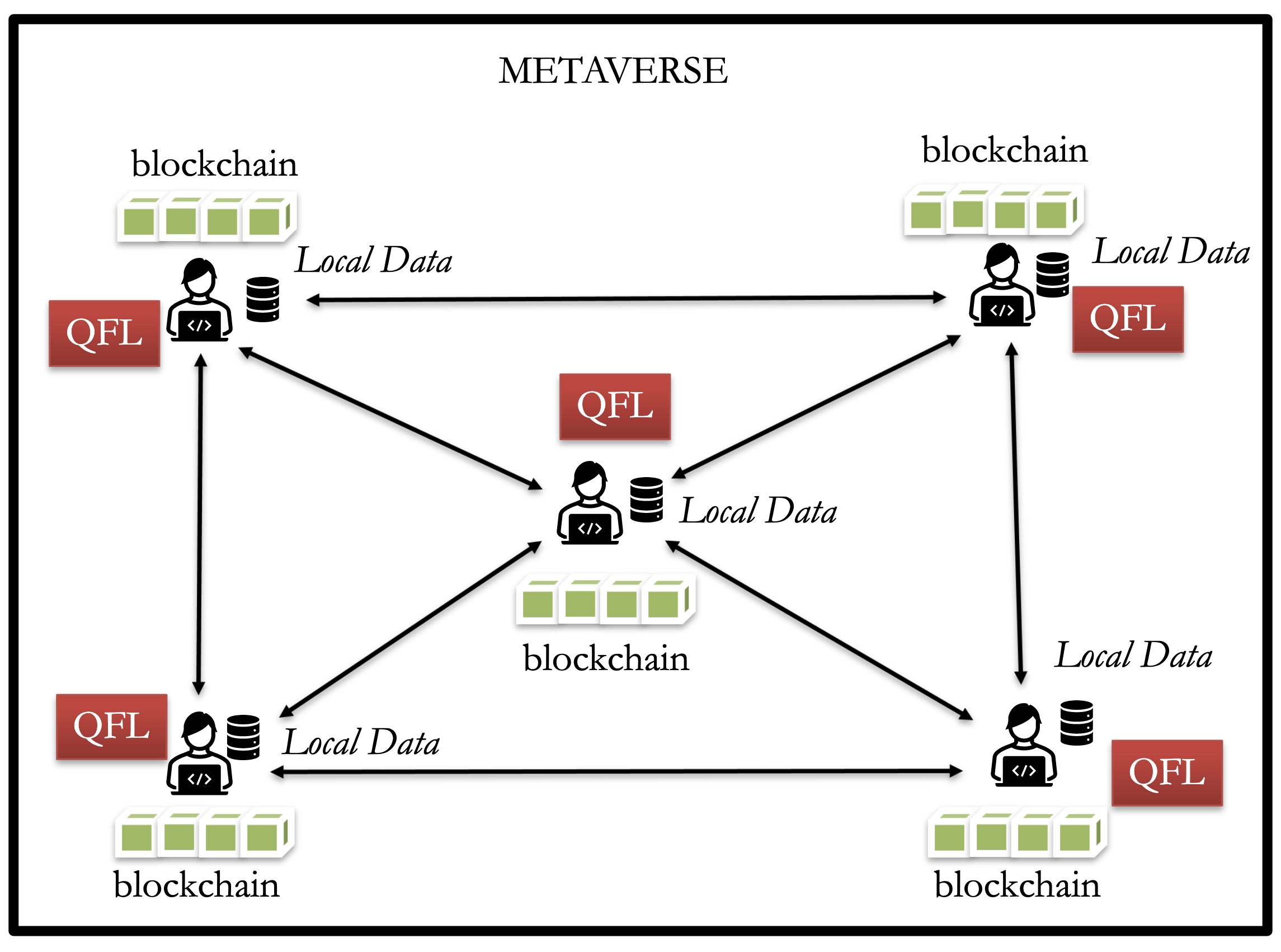} 
    \caption{Peer-to-Peer Blockchain-based Quantum Federated Learning (BQFL) integration into Metaverse. 
    Such a P2P model with blockchain improves reliability
    and trustworthiness considerably.}
    \label{fig:bqfl}
\end{figure}

On the other hand, the blockchain has successfully guaranteed the immutability and trustworthiness of FL~\cite{pokhrelFederatedLearningBlockchain2020, pokhrelBlockchainBringsTrust2021} while
facilitating decentralized financial transactions (e.g., cryptocurrency). 
It is essential to investigate the potential of blockchain for decentralizing
QFL functionalities. However, the greenfield development of QFL and decentralizing its capabilities are nontrivially challenging tasks.

Despite these challenges, we must overcome them so that the developed QFL not only becomes a part of our future but molds it, propelling us toward a more environmentally and technologically sophisticated civilization. The emerging Metaverse can be taken as an example of such a civilization, which aims to bring together multiple sectors into one
ecosystem and create a virtual environment that mirrors its natural equivalent and maintains persistence. 
Taking the metaverse as a motivating example, we
aim to build QFL capabilities to support an interactive
a platform that blends social networking, gaming, and simulation to
create a virtual space replicating the real world. 
Furthermore, such QFL can facilitate collaborative learning, which is very important in the metaverse.
\subsection{Motivation and Background}
In
this work, our primary focus is on designing a robust blockchain-based QFL (BQFL) framework that is specifically tailored to support
Metaverse.
In addition to that, we also shed light on the hybrid metaverse.
As illustrated in Figure~\ref{fig:bqfl},
our objective is to develop a peer-to-peer blockchain-based QFL framework suitable for the Metaverse.

One main motivation for this work is the increasing interest in 
the growing demand for secure, decentralized, and 
trustworthy ML algorithms in the Metaverse.
A QFL based on a centralized global server is prone
to single-point failure issues.
To address this issue, in this work, we propose a
decentralized QFL protocol that takes advantage of the
immutable and distributed nature of blockchain technology to create a more trustworthy
QFL framework.
Eventually, the design of decentralized QFL
ensures the security and trustworthiness of machine learning
algorithms and cryptocurrency transactions in the Metaverse.

\subsection{Related Works}

QFL is an emerging area with several studies in recent years \cite{larasatiQuantumFederatedLearning2022, kaewpuangAdaptiveResourceAllocation2022a,yangDecentralizingFeatureExtraction2021a,yunSlimmableQuantumFederated,xiaQuantumFedFederatedLearning2021a, xiaDefendingByzantineAttacks2021, zhangFederatedLearningQuantum2022}.
Most QFL works focus on optimization
\cite{kaewpuangAdaptiveResourceAllocation2022a,yangDecentralizingFeatureExtraction2021a,yunSlimmableQuantumFederated,xiaQuantumFedFederatedLearning2021a}.
Some works are for security aspects \cite{xiaDefendingByzantineAttacks2021, zhangFederatedLearningQuantum2022, gurungSECURECOMMUNICATIONMODEL2023} and some in terms of implementation \cite{qiFederatedQuantumNatural2022, yamanyOQFLOptimizedQuantumBased2021,abbasPowerQuantumNeural2021, huangQuantumFederatedLearning2022, chehimiQuantumFederatedLearning2022}.
BCFL is well studied in the literature; see \cite{pokhrelFederatedLearningBlockchain2020, pokhrelFederatedLearningMeets2020, chenRobustBlockchainedFederated2021, pokhrelDecentralizedFederatedLearning2020} etc. Pokhrel \textit{et 
al.}~\cite{pokhrelFederatedLearningBlockchain2020} introduced the BCFL design to
address privacy concerns and
communication cost in vehicular networks.
Whereas, Chen \textit{et al.}~\cite{chenRobustBlockchainedFederated2021} 
addressed the single point of failure and malicious device detection through the distributed validation mechanism. 
 These works~\cite{zhaoExactDecompositionQuantum2022a, 
 huangQuantumFederatedLearning2022, chehimiQuantumFederatedLearning2022, 
 pokhrelFederatedLearningBlockchain2020, chenRobustBlockchainedFederated2021}
have built enough ground to investigate BQFL, to harness the benefits of both blockchain and QFL, which is the main research problem considered in this paper.

Duan \textit{et. al} \cite{duanMetaverseSocialGood2021b} proposed a three-layer
metaverse architecture consists of infrastructure, interaction, and
ecosystem.
Metaverse is based on blockchain and has been studied in a few types of
literature \cite{yangFusingBlockchainAI2022a}, but the use and integration of machine learning, especially QFL, in Metaverse, has not yet been explored.
Bhattacharya \textit{et. al}
\cite{metaverseP2PGaming} proposed  FL integration metaverse for the gaming environment.
Chang \textit{et. al}
\cite{chang6GEnabledEdgeAI2022} provided a survey on AI-enabled by 6G for Metaverse.
Zeng \textit{et. al}
\cite{zengHFedMSHeterogeneousFederated2022} proposed a high-performance and efficient FL system for Industrial Metaverse.
The authors partly incorporated a new concept of Sequential-to-parallel (STP) training mode with a fast Inter-Cluster Grouping (ICG) grouping algorithm and claim that it effectively addresses the heterogeneity issues of streaming industrial data for better learning.

\subsection{Contributions}
In summary, the main contributions of this paper are as follows.
\begin{enumerate}
    \item We analyze and develop a novel trustworthy blockchain-based quantum federated learning (BQFL), i.e., a decentralized approach to QFL, by presenting and combining the principles of blockchain, quantum computing, and federated learning. 
    \item We implement BQFL and develop new insights into the feasibility and practicality of BQFL. In addition, we develop substantial reasoning with a thorough theoretical analysis in terms of employing BQFL under Metaverse, considering both privacy and security concerns. 
\end{enumerate}
\section{Identified Research Challenges}
We have identified several bottleneck challenges in
the field of classic and quantum FL and their potential
applications for orchestrating Metaverse, all of which are explained in the following.
\begin{enumerate}
    \item \textit{Limitations of CFL}:
    One of the key limitations of conventional CFL, which is used to
    train models across heterogeneous clients and aggregate them at a
    the central server is its restricted computational power 
    \cite{kwakQuantumDistributedDeep2022}. This can undermine the advantages of decentralized learning, especially when clients are a mix of pioneers and stragglers. To address this limitation, the use of QNNs is required, which is known as QFL \cite{abbasPowerQuantumNeural2021}. However, the literature currently lacks a complete analysis and understanding of QNNs over FL, and more research is needed in the field of BQFL to investigate how QNNs work in various FL scenarios, such as blockchain integration and data availability.
    
    \item \textit{Central-Serverless  QFL}:
    In QFL, only a central server typically aggregates the model parameters, leading to single points of failure and a lack of incentive mechanisms that can limit the performance of QNNs. 
    To address these issues, blockchain technology can be highly effective in introducing trust into QNNs due to its immutable nature. However, the architecture of blockchain-based QNNs is poorly understood in the literature, and we aim to investigate this by developing a BQFL framework that can resolve these issues.
    \item \textit{Challenges of P2P Blockchain QFL}:
Implementing the P2P blockchain FL is a complex and resource-intensive task requiring significant network bandwidth and computational power.
    As a result, extensive research work and studies are required in this direction.
    \item \textit{External use of Blockchain for QFL}:
    This approach will make the implementation simpler and more scalable than the P2P approach. However, it will be less decentralized and trustless than the P2P approach, as it might need a central server.
    \item \textit{Challenges in the development of Metaverse}:
Progress in technologies such as Virtual Reality and Augmented Reality (VR/AR) has led to the possibility of the existence and development of Metaverse. 
    However, problems such as the digital economy's transparency, stability, and sustainability cannot be solved with these technologies alone \cite{duanMetaverseSocialGood2021b}.
    Also, one of the main challenges or shortcomings in the development of Metaverse is the lack of a clear and distinctive architectural definition that could be used as a standard blueprint approach.
    Some challenges in the development of Metaverse can be:
    \begin{enumerate}
        \item \textit{How to cope with the computational requirement of Metaverse?}
         \item \textit{How to achieve efficient resources allocation and 
         solve large-scale data complexities?}
        \item \textit{Can QFL and Blockchain together solve the 
        issue?}
    \end{enumerate}
    \item \textit{Blockchain QFL for Metaverse}:
    BQFL is a decentralized approach to federated learning that 
    combines blockchain technology, quantum computing, and machine learning.
    On the other hand, Metaverse is a blockchain-based virtual world that allows users to interact with each other along with their digital assets and experiences, such as Virtual land where one can buy, purchase, and build businesses on the land.
    Blockchain technology can facilitate trusted digital ownership, interoperability, and decentralization to improve user experience and create new business models.
    
    \item \textit{Problem with today's quantum computers}: 
    Current quantum computers are a sort of proof-of-concept that the technology can be built \cite{preskillQuantumComputingNISQ2018}.
    But the problem is their infeasibility in real-life applications and the frequent error that occurs in them, referred to as 'noise' in terms of computation.
    Thus, this leads to the need for a thorough study and investigation of the implementation of how the blockchain works along with QFL networks.
    
    \item \textit{Multi-Model AI for Metaverse}:
    Metaverse requires different types of model learning, not just limited to text recognition. 
    Thus, the model training needs to learn different forms of data images, speeches, videos, etc. 
    One of the current works being done towards it is by Meta AI, which refers to as self-supervised learning \cite{baevskiData2vecGeneralFramework}.
    Also, a world model is needed specifically for Metaverse because it needs to be able to interpret and understand different forms of data such as text, video, image, etc.

    \item \textit{Hybrid Metaverse}:
     It won't be an overstatement to say that Metaverse will be our future in some way or another for sure.
    However, the metaverse, especially as proposed by Meta (former Facebook), has come under a number of criticisms that doubt its future.
    There are many limitations to the purely virtual metaverse.
    First, not everyone would love to be stuck in a virtual world 24/7. 
    Thus, architecture design in the form of a hybrid metaverse that can incorporate both AR and VR simultaneously needs to be redesigned.
    Thus, the question is, can we create a metaverse that is a replica of the real world where an actor in that environment can quickly switch between the virtual (VR) and semi-real (AR) world simultaneously?
    For example, suppose that a replica of an existing shopping center is created as a shopping center metaverse. 
    An actor/actress within VR equipment can enter the metaverse replica.
    However, while doing so, is it possible for her/him to be
    teleported to the actual store that s/he wants to visit and see items, in real interaction with the people in the shopping center as if s/he visits there?

   \item \textit{Data Utilization}:
   Most metaverse today focuses on avatar creation, with avatar actions limited to interaction. 
   Thus, with so many actors and devices working together, 
   the data thus generated must be utilized for an end-to-end solution \cite{metaverseP2PGaming}.
   The metaverse involves sensor data, such as the user's physical movement and motion capture, and personal data, like biometric data, etc. which are highly sensitive and personal.

   \item \textit{Full potential of Metaverse}:
   Due to limited resources and computing power,  Metaverse is still far from reaching its full potential of total immersion, materialization, and interoperability \cite{chang6GEnabledEdgeAI2022}.
   
\end{enumerate}


\section{Preliminaries, Theories, and Ideas}
In this section, 
we cover fundamental concepts of quantum machine learning and present
the theoretical concept behind the implementation of the  system framework proposed \cite{zhaoExactDecompositionQuantum2022a}.
\subsection{Terms and Terminologies}
Qubit is the fundamental unit of data storage in quantum computing \cite{QuantumDataBaidu}. 
Unlike classical computers, which use bits with only two values (0 and 1), 
qubits can also take on a range of values due to mechanical superposition. 
The number of qubits needed depends on the computational problems that need
to be solved.

The quantum channel refers to the medium used for transferring quantum
information (qubits) \cite{zhaoExactDecompositionQuantum2022a}. Quantum Federated Averaging aims to find a quantum channel that takes an input state and transforms it into the desired output.
Quantum Classifiers are devices that solve classification problems.
The quantum circuit takes the quantum state as an input. 
Tensor Circuit \cite{zhangTensorCircuitQuantumSoftware2022a} is an
open-source quantum circuit simulator that supports different features such as automatic
differentiation,
hardware acceleration, etc. 
It is especially useful for simulating complex quantum
circuits used in variational algorithms that rely on parameterized quantum
circuits.

Noisy Intermediate-Scale Quantum (NISQ) computers \cite{preskillQuantumComputingNISQ2018} with a limited
number of error-prone qubits are currently the most advanced quantum
computers available \cite{yunSlimmableQuantumFederated2022}. 
Quantum computers that are fully fault-tolerant and capable of
running large-scale quantum algorithms are not available at the moment.
Since real quantum computers are not easily accessible, quantum circuit
simulation on classical computers is necessary. The tensor circuit library is
commonly used for this purpose.

The quantum neural network is a variational quantum circuit used in quantum
computing. 
Variational quantum circuits (VQC) \cite{arthurHybridQuantumClassicalNeural2022, cerezoVariationalQuantumAlgorithms2021} are a technique that mimics
classical neural networks in quantum computing. It involves training a
dataset, encoding the quantum states as input, producing output quantum
states, and then converting the output back into classical data.

Data Encoding \cite{jerbiQuantumMachineLearning2023} is the process of transforming classical information into
quantum states that can be manipulated by a quantum computer. 
Amplitude encoding stores data in the amplitudes of quantum states, while
binary encoding stores information in the state of a qubit. Binary encoding
is preferable for arithmetic computations, while analogue encoding is
suitable for mapping data into the Hilbert space of quantum devices.

Quantum Convolutional Neural Network (QCNN) \cite{yangDecentralizingFeatureExtraction2021a} is a type of neural network
used in quantum computing.
Quantum perceptron is the smallest building block in Quantum Neural Networks (QNNs) \cite{gargAdvancesQuantumDeep2020}. 

Blockchain is a decentralized ledger system that
relies on distributed nodes to keep a record of transactions that cannot be
altered once committed.
Smart contracts are programs that automatically execute and follow the terms
of a contract or agreement. Blockchain consensus is a crucial task that
ensures the system's overall reliability and addresses unexpected
behavior from clients or malicious nodes on the network.
Decentralized Applications run autonomously on decentralized computing systems such as the blockchain. In peer-to-peer networks, each
the participating device has equal privileges and the distributed network
architecture is embraced.

\subsection{Design Ideas and Fundamentals}
\subsubsection{Gradient Descent}
To minimize a function $f(\overrightarrow{w})$ and its gradient $\delta 
f(\overrightarrow{w})$ starting from an initial point, the optimal approach
is to update the parameters in the direction of the steepest descent, given
by $\overrightarrow{w}_{n+1} = \overrightarrow{w}_n - \eta \Delta 
f(\overrightarrow{w})$. This process can be repeated until the function reaches a
local minimum $f(\overrightarrow{w}^*)$. Here, $\eta$ is the step size or
learning rate, which handles the magnitude of the update at each iteration.

\subsubsection{Local Training}
Initially, the first trainer accesses the global parameters $w_g$. 
Then, the Adam optimizer, a popular optimization algorithm used in machine
learning, is utilized through the \textit{optax} library for gradient-based optimization. 
The learning rate is set to a commonly used value of $1e-2$.
Next, the optimizer state $opt\_state$ is initialized with the current
parameters $w_d$. This state represents the optimizer's internal variables
and state, which are used to update the model parameters during training. 
The optimizer state is updated during the training process, which involves
iterating over the training data for a certain number of epochs, $epochs$.
In each iteration, the loss value $loss\_val$ and the gradient value $grad\_val$
are calculated for the current batch using the model parameters $w_d$, input
data $x$, output data $y$ and variable $k$. 
The optimizer state $opt\_state$ is updated using the calculated gradients
$grad\_val$ with the current parameters $params$, and the updated values
$updates$ are stored. These updates are then applied to the current model
parameters.
Finally, the mean loss $loss\_mean$ for the current batch is calculated using
$loss\_val$, which is a list of individual loss values for each example in the
batch.

\subsubsection{Class filter [filter function]}
To remove certain labels from a given dataset, one approach is to iterate
over the dataset and filter out samples with undesired labels. 
This can be accomplished using a conditional statement to check if each
sample's label is in the list of labels to be removed. 
If a sample's label is found in the list, it can be skipped or removed from
the dataset.

\subsubsection{Quantum Circuit [clf function]}
To create a quantum circuit, we define a function that applies $k$ layers of
quantum gates to a given circuit $c$. Each layer consists of gates applied to
each qubit in the circuit. To do this, we iterate over each layer first and then over each qubit in the circuit.
If the 2D numpy array $params$ represents the circuit parameters,
the rotation angles for each qubit are determined based on the corresponding
parameters at a particular index. To create an entangled state, a Controlled 
NOT (CNOT) gate is applied to each neighboring pair of qubits in the circuit.
Next, each qubit is rotated about the x-axis with a rotation angle determined by
the corresponding parameter at index [3 * j, i]. 
Then, a z-axis rotation is applied to each qubit with a rotation angle determined by the
corresponding parameter at index [3 * j + 1, i]. 
Finally, another x-axis rotation is applied to each qubit with an angle of rotation determined by
the parameters at index [3 * j + 2, i].
Here, $j$ represents the layer number, while $i$ refers to the qubit number. 
By applying this series of gates to the circuit, we can create a quantum state with desired
properties. 
It is important to note that the specific combination of gates used can have a significant
impact on the resulting quantum state i.e. a careful selection of gates and parameters is critical 
for achieving desired outcomes.

\subsubsection{Parameterized Quantum Circuits (PQC) }
When training a parameterized quantum circuit model, the objective is to
learn an arbitrary function from the data. This is done by minimizing a cost or
loss function, denoted as $f(\overrightarrow{w})$, with respect to the
parameter vector $\overrightarrow{w}$. The process involves minimizing the 
expectation value, $\bra{\psi(\overrightarrow{w})}\hat{H}\ket{\psi(\overrightarrow{w})}$, where 
$\hat{H}$ is the Hamiltonian of the system.
To achieve this, the trainers first send the parameters $w_n$ to the server. 
Then, the expectation value is computed as $\bra{\psi(w_n)}\hat{H}\ket{\psi(w_n)}$. 
Parameters are updated to $w_{n+1}$, and the process is repeated until
convergence.
Gradient-based algorithms are commonly used to optimize the parameters of a
variational circuit, denoted $\mathbb{U_w}$. 
For a PQC, its output is a quantum state $\ket{\psi(w_n)}$, where $w_n$ is a vector of parameters that can be tuned
\cite{zhangTensorCircuitQuantumSoftware2022a}.

\subsubsection{Prediction probabilities [readout function]}
The purpose of this function is to extract probabilities from a given quantum circuit. 
The function takes a quantum circuit $c$ as input and generates probabilities using one of two modes, namely
softmax and sampling method.
In "softmax" mode, the function first computes the logits for each node in the neural network, 
which are the outputs of the last layer before applying the activation function. 
These logits are then used to compute the softmax probabilities.
On the other hand, if the "sample" mode is selected, the function computes the wave-function
probabilities directly and then normalize them to obtain the output probabilities.

\subsubsection{Loss Function}
A neural network loss function is optimized in a quantum circuit taking four input arguments:
network parameters $params$, input data $x$, target data $y$, and the number of quantum circuit layers $k$. 
First, a quantum circuit with $n$ qubits is constructed using the input data $x$. 
The circuit is then modified by a classifier function that takes input arguments $params$, $c$, and $k$, thus transforming the circuit into a quantum neural network. 
The modified circuit is then passed to a read-out function that produces predicted
probabilities for each input. Finally, the loss is computed as the negative logarithmic likelihood
of the predicted probabilities and is averaged over all samples in the batch.

\subsubsection{Accuracy Calculation}
To evaluate the accuracy of the quantum classifier model with given parameters $params$, 
input data $x$, target labels $y$ and a number of layers $k$, the following steps are taken:
First, a quantum circuit $c$ is created using the input data $x$.
Then, the function $clf$ is applied to the circuit $c$ with the parameters $params$ to update the
circuit.
The updated circuit is then passed to the $readout$ function to obtain the predicted probabilities
for each label.
Then, the highest probability index is obtained for each input in $x$ 
and then compared with the true class label index for each input in $y$.
Finally, the precision is calculated by dividing the number of correct predictions
by the total number of inputs in $x$.

\begin{figure}[t]
    \centering
    \includegraphics[width=0.8\columnwidth]{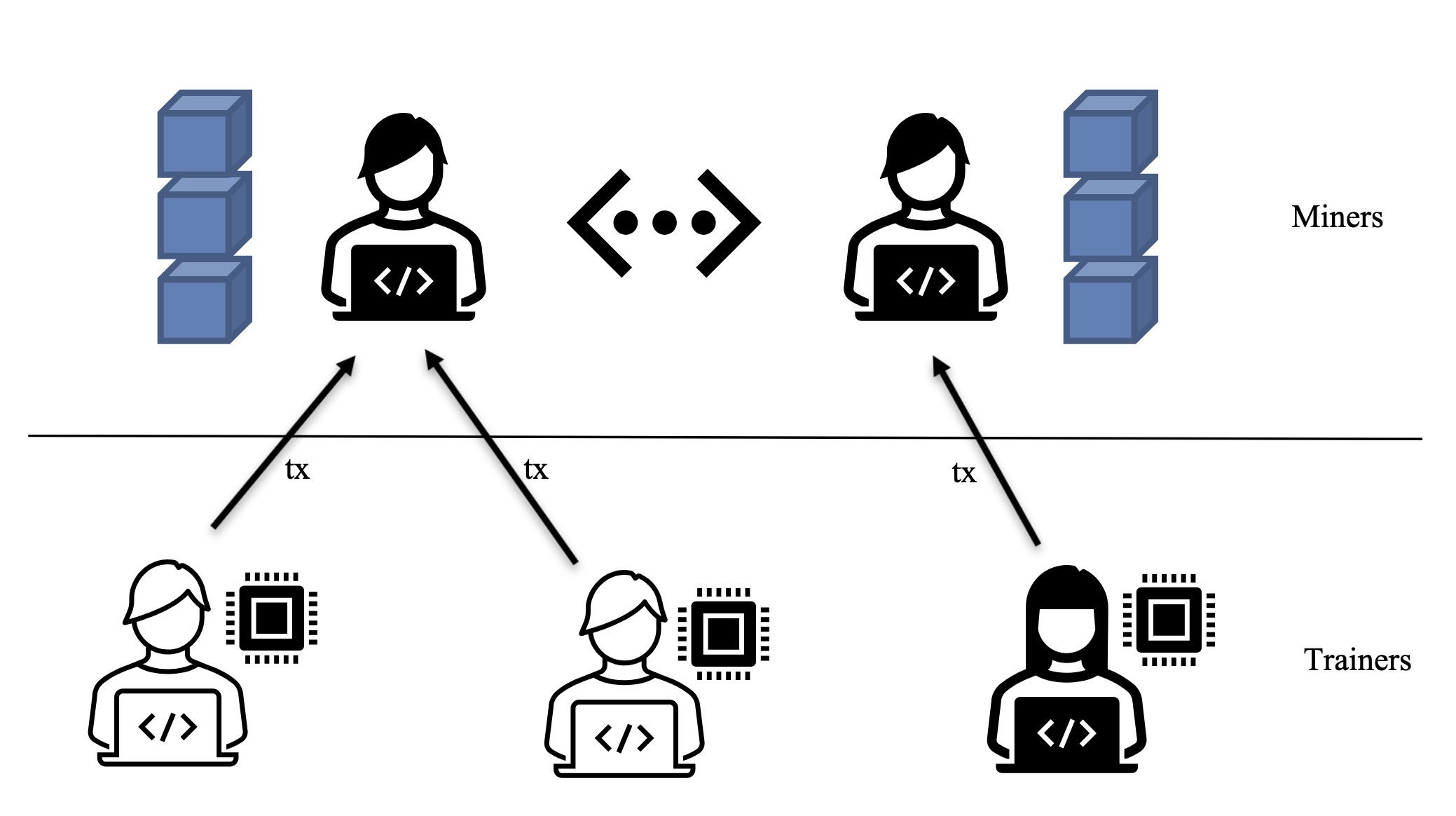}
    \caption{Blockchained QFL}
    \label{fig:BQFL}
\end{figure}

\section{Proposed BQFL Framework} 
\label{sec:proposed}

\begin{figure*}
    \centering
    \includegraphics[scale=0.46]{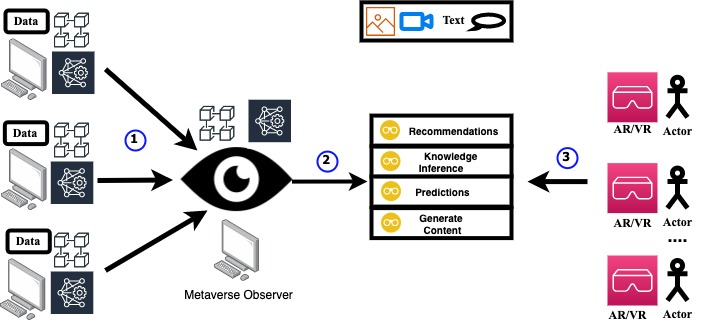}
    \caption{Metaverse Powered by BQFL. 1) BQFL training to create a world model for Metaverse Observer, 2) World Model used for different purposes, 3) Actors in Metaverse using AR/VR technology to access Metaverse.}

    \label{fig:metaverse}
\end{figure*}
We have proposed a trustworthy Blockchain-based QFL (BQFL) framework 
that integrates blockchain with QFL to address various issues related
to privacy, security, and decentralization in machine learning. 
The framework consists of two approaches, one where the blockchain is separate from QFL, 
and the other where the blockchain is within QFL in a completely peer-to-peer network.
\subsection{Motivating example: Metaverse with BQFL}
Consider a Metaverse space for a city center replicated in a digital twin creating a hybrid space for visiting shops, 
shopping, looking at items closely, and meeting people.
Unlike a complete virtual immersive experience, 
this metaverse could be a hybrid one that has options for complete immersiveness into the virtual world as well as in augmented reality in the actual store.
In this way, we can respond to people with AR by pushing a button or some specific method. 
For this purpose, the whole system network must perform extremely fast with the least delay issues. 
Thus, in this regard, as we have found out from experimental and theoretical aspects, BQFL indeed performs better.
\subsection{The BQFL Framework Design}
As shown in Figure \ref{fig:metaverse}, 
the proposed framework consists of an actual physical space and a virtual replica of the real-world space.
The three main building blocks are the QFL framework, Blockchain, and Metaverse. 
For QFL, we have nodes or devices with QNN for training with the local data. 
Once all local devices do the local training, then the final averaging of the local models for final global model generation can be done by Metaverse Observer or done in a pure decentralized fashion. 
The Metaverse observer is the main entity in the framework which looks after overall activities in the metaverse. 
This entity could be designed for each type of specific part in the metaverse to perform specific tasks.
For example, with the proposed framework, the main task for the metaverse observer is to provide knowledge inference, prediction, recommendation, etc. to the actors.
The actors or people participating in the metaverse use AR/VR technology to enter into the virtual (purely virtual world with avatars) or augmented reality version of real physical space.
The overall workflow for the framework is presented in Algorithm \ref{alg:qfl_metaverse}, whereas, the overall local training in the higher level representation is presented in Algorithm \ref{alg:qfl_blockchain}.

\begin{algorithm} [h!]
\caption{Hybrid Metaverse with BQFL powered Observer}
\label{alg:qfl_metaverse}
\begin{algorithmic}[1]
\Procedure{\textcolor{blue}{Initialization}}{}
\State Define $N$ number of nodes for QFL.
\State Hybrid Metaverse environment with both AR and VR. 
\State $m$ actors and 
Metaverse Observer, $\mathbb{O}$
\State Blockchain network.
\State VR/AR replica of the real world.
\State Actors in Metaverse, \{$\mathbb{A}_i\}$ use AR/VR technology.
\EndProcedure

\Procedure{\textcolor{blue}{Actor Activity}}{}
\State Access replica metaverse using AR or VR headset.
\State Meet people with avatars and interact.
\If{AR is chosen}
\State Go to the actual shopping center (teleportation)
\Else
\State Stay in VR world.
\EndIf

\EndProcedure

\Procedure{\textcolor{blue}{Device Training}}{}
\For{device in devices}
\State Train Local parameters.
\State Send learned parameters for FedAvg.
\State Retrieve global model after FedAvg.
\State 
\EndFor
\EndProcedure

\Procedure{\textcolor{blue}{Metaverse Observor}}{}

\For{device in devices}
\State Observer performs various tasks.
\State Recommendation System.
\State Actor activity detection.
\If{Actor activity legal}
\State Pass
\Else
\State Stop Actor.
\EndIf
\EndFor
\EndProcedure

\end{algorithmic}
\end{algorithm}

\subsection{Different Approaches}
\subsubsection{Blockchain separate from QFL - Blockchain externally}
Blockchain can be integrated into QFL in various ways depending on the need for decentralization. 
One approach is to use blockchain only for storing transactions such as rewards, model weights, 
etc., while the clients that train the model do not have copies of the blockchain. 
In this approach, the blockchain nodes can act as miners, and the QFL nodes can be different from the miners. 
This approach of using the blockchain externally is a secure and transparent way to
store model updates and other metadata. 
The blockchain is not used for coordination or
communication between clients.
The steps involved in integrating blockchain externally in QFL are as follows:
\begin{enumerate}[a.]
\item Set up a blockchain network with nodes to store and validate transactions.
\item Deploy a smart contract.
\item Submit model updates to the smart contract after each round of training by the clients.
\item The smart contract aggregates the updates and generates a global model for the next round of training.
\end{enumerate}

\subsubsection{Blockchain within QFL- Peer to Peer BQFL}
In another approach, we consider decentralized QML in a completely peer-to-peer network
where each client has a copy of the blockchain. 
We present a decentralized peer-to-peer network QFL for demonstration and experimental purposes.
The steps involved in integrating blockchain within QFL are as follows.
\begin{enumerate}[a.]
\item Clients communicate with each other through the blockchain network.
\item Blockchain is used to facilitate communication and coordination between clients, along with the storage of updates and other metadata. 
\item Each client in the network validates transactions and contributes to the consensus mechanism, ensuring the integrity and security of the system.
\end{enumerate}

This QFL P2P blockchain will allow for a fully decentralized and trustworthy system where clients
communicate and collaborate directly without any central authority or third-party service. 
Some of the obvious advantages of this approach are greater privacy, security, and transparency,
with reduced risk of a single point of failure or attack.

\subsubsection{Foundations of BQFL for Metaverse}
The architecture of the proposed BQFL consists of a quantum computing infrastructure, the QFL Algorithm, and Metaverse.

\begin{enumerate}[a.]
\item Quantum Computing Infrastructure:
To train machine learning models using QFL, we require quantum computing resources that are capable of running QFL algorithms.
\item QFL Algorithm:
The QFL algorithm should work in a distributed and privacy-preserving manner among multiple users contributing their local data for the training process.
\item Metaverse:
Metaverse is a virtual world where clients can collaborate and communicate with each other.
\end{enumerate}
 \subsubsection{Assumption for the framework design}
 We have made the following assumptions.

\begin{enumerate} [a.]
    \item FedAvg performs $E$ steps of SGDs in parallel on a set of devices. Opposite to QFL with a central server, 
    model averaging occurs without a central server architecture.
    \item In the presence of stragglers, devices that can become inactive are inevitable. Thus, all devices are assumed to be active throughout the process.
    \item Data sets are non-IID.
\end{enumerate}
 

\begin{algorithm}    
\caption{Blockchain QFL}
\label{alg:qfl_blockchain}
\begin{algorithmic}[1]
\Procedure{\textcolor{blue}{Initialization}}{}
\State Total $n$ devices, $\{d_1, d_2, d_3, ... , d_n\}$.
\EndProcedure

\Procedure{\textcolor{blue}{deviceTraining}}{{$pk, sk$}}
\For{device $d_i$ in $\{d_i\}$}
\State Encode data.
\State Resize data to fit the quantum circuit.
\State Remove data classes.
\State Train local params $w_i$ with QNN.
\State Send \{$w_i$\} to FedAvg.
\EndFor
\EndProcedure

\Procedure{\textcolor{blue}{FedAvg}}{}
\State Recieve trained model from the nodes.
\State Create a global model.
\EndProcedure
\end{algorithmic}
\end{algorithm}

\section{Theoretical Analysis}
In this section, we present theoretical analysis in terms of convergence for blockchain-based quantum federated learning.
\subsection{Convergence Study}
In this section, we examine the convergence properties of the proposed 
algorithms. 
Here, we analyze the convergence and conditions for convergence under different
assumptions about the data distribution and communication patterns.
From \cite{liConvergenceFedAvgNonIID2020}, we follow these assumptions.
\begin{assumption}
Objective functions $\Phi_1, \ldots, \Phi_N$ are all $L$-smooth, which is a standard assumption in federated learning. 
\begin{equation}
\Phi_k(\beta ) \leq \Phi_k(\theta) + (\beta  - \theta)^T \nabla \Phi_k(\theta) + \frac{L}{2} ||\beta  - \theta||^2_2.
\end{equation}
\end{assumption}
where, for any vectors $\beta $ and $\theta$, the function value at $\beta $ is upper-bounded by the function 
value at $\theta$, plus a term that depends on the gradient of $A_k$ at $\theta$ and the distance between $\beta $ and $\theta$. 

\begin{assumption}
    The objective functions $\Phi_1, \dots, \Phi_N$ are all $\mu$-strongly convex. 
    This means that for all $\beta $ and $\theta$, the following inequality holds:
    \begin{equation}
    \Phi_k(\beta ) \geq \Phi_k(\theta) + (\beta -\theta)^T \nabla A_k(\theta) + \frac{\mu}{2}||\beta -\theta||^2_2.
    \end{equation}
   
\end{assumption}
 where, $k \in \{1,\dots,N\}$, $\beta$ and $\theta$ are vectors in the same space as the gradients, 
 and $\mu$ is a positive constant that controls the strength of the convexity of the functions.

 \begin{assumption}
     Let suppose $\xi_k^t$ be sampled uniformly at random from the local data of the $k^{th}$ device. 
     The variance of the stochastic gradients in each device is bounded by a constant $\sigma_k^2$ i.e.
     
\begin{equation}
E ||\nabla \Phi_{k}\left(\theta_{t}^k, \xi_t^k\right)-\nabla \Phi_{k}\left(\theta_{t}^k\right)||^{2} \leq \sigma{k}^{2}, \text { for } k=1, \ldots, N.
\end{equation}
 \end{assumption}

Here, $\nabla \Phi_k(\theta_{t}^k, \xi_t^k)$ represents the stochastic gradient of the $k^{th}$ device's 
objective function with respect to its local model parameter at iteration $t$, evaluated at the 
random data sample $\xi_k^t$. 
$\nabla \Phi_k(\theta_{t}^k)$ represents the average stochastic gradient of
the $k^{th}$ device's objective function with respect to its local model parameter at iteration $t$ 
evaluated on all the local data samples of the device. 
 \begin{assumption}
 For stochastic gradients, its expected squared norm is uniformly bounded, i.e., 
     \begin{equation}
         \mathbb{E}||\nabla \Phi_k(\theta_{t}^k, \xi_t^k)||^2 \leq G^2 \text{ for all } k=1,\dots,N \text{ and } t=1,\dots,T-1.
     \end{equation}
 \end{assumption}
 where, $A_k$ is the objective function of the $k^{th}$ client, 
$\theta^k_t$ is the local model parameter of the $k^{th}$ client at iteration $t$, 
$\xi^k_t$ is the random data sample from the $k^{th}$ client's local data at iteration $t$, 
$\Delta \Phi_k(\theta^k_t, \xi^k_t)$ is the stochastic gradient of the $k^{th}$ client's objective function with 
respect to its local model parameter at iteration $t$, evaluated at the random data sample $\xi^k_t$.
$G$ is the bound on the expected squared norm of stochastic gradients.

Then,  from \cite{liConvergenceFedAvgNonIID2020} if assumptions 1-4 hold, then
fedAvg satisfies,
\begin{equation} \label{eq:convergence_fedAvg}
 \mathbb{E}[\Phi(\theta_T)] - \Phi^* \leq \frac{\kappa}{\gamma + (T-1)}\left(\frac{2B}{\mu} + \frac{\mu \gamma}{2} E||\theta_1 - \theta^*||^2\right),
\end{equation}
where, 
\begin{align*}
     B = \sum_{k=1}^N p_k^2 \sigma_k^2 + 6L\Gamma + 8(E-1)^2G^2.
\end{align*}

In equation \ref{eq:convergence_fedAvg},
$\mathbb{E}[\Phi(\theta_{T})] - \Phi^*$ represents the expected excess risk, $B$ is a constant term, 
and $p_k, \sigma_k, L, \Gamma$, and $G$ are parameters that depend on the problem and the algorithm.
Here, $\kappa$ and $\gamma$ are constants defined as $\kappa = \frac{L}{\mu}$ and $\gamma = \max \{ 8\kappa,
E \}$. Whereas,  $\eta_t = \frac{2}{\mu(\gamma + t)}$. 

\subsection{Considering encoding and decoding time for QFL}
Data encodings for quantum computing is the process of data representation for the quantum state of a system \cite{weigoldEncodingPatternsQuantum2021}.
This process of encoding classical data into a quantum state can be a 
significant bottleneck, as it can introduce significant time delays. 
The choice of encoding scheme 
used will depend on the specific task and the available hardware. 
Here, we consider a vanilla data encoding method which is known as "amplitude encoding".
It involves mapping classical data to the amplitudes of a quantum state. 
Given an input data vector of length $L$, the quantum state can be represented as:
$$|\psi\rangle = \sum_{i=1}^{L} a_i |i\rangle.$$
Here, $a_i$ is the $i$th element of the input data vector, and $|i\rangle$ is the basic state
corresponding to the binary representation of $i$.
A sequence of gates is applied to encode the data vector into a quantum circuit.
This sets the amplitudes of the quantum state to the values in the input data vector.
In order to do that, a set
of rotation, gates can be used that adjust the phase of each basis state, followed by a set of controlled-NOT (CNOT) gates for the amplitudes.
Even though, the time required to apply a single gate in a quantum circuit is typically on the order of 
nanoseconds or less, 
but for the whole data vector, it will depend on the number
of qubits and the complexity of the encoding process.
Assuming that we have a quantum circuit with $n$ qubits, the time required to encode a single
element of the input data vector can be approximated as:
$$t_{\text{i}} \approx n \cdot t_{\text{gate}}$$
Here, $t_{\text{gate}}$ is the time required to apply a single gate in the quantum circuit. 
Therefore,  for the entire input data vector, the total time required to encode the  can be approximated as:

\begin{equation} \label{eqn:encoding_time}
    t_{\text{T}} \approx L \cdot n \cdot t_{\text{gate}}
\end{equation}

The above equation  follows the assumption that we encode each element of the input data vector one at a time. 

\subsection{Blockchain Time Delay}
Blockchain time delay can include both communication delay and consensus delay.
Also, the time required to share the new copy of the blockchain ledger with each other can be added to the total delay time.
We also need to consider the time it takes for a block to be appended to the blockchain after it is proposed by a node as well as the block propagation delay.

Suppose we have a blockchain network with $n$ nodes, and each node has a stake $Stake_i$ and a
probability $prob_i$ of being selected as the next validator to create a block. 
The probability of a
the node $i$ being selected as the validator can be represented as,

$$prob_i = \frac{Stake_i}{Stake}$$

Let's assume that each node takes $t$ seconds to create a block and the network latency is $L$ 
seconds. 
The time it takes for a block to be created and validated is:

$$T = \max(t, L) + t$$


Now,  to calculate the expected time it takes for a block to be created and
validated by the network, given the stake and probability of each node, we can write,
\begin{equation}\label{eqn:blockchain_time}
    E[T] = \frac{1}{\sum_{i=1}^n prob_i} \sum_{i=1}^n prob_i T
\end{equation}


\begin{theorem} \label{theorem:main_convergence}
The proposed blockchain-based quantum federated learning algorithm satisfies:

\begin{align} 
E[total\_time] \leq  \frac{\kappa}{\gamma + (T-1)}\left(\frac{2B}{\mu} + \frac{\mu \gamma}{2} E||\theta_1 - \theta^*||^2\right) + \nonumber\\
\frac{1}{\sum_{i=1}^n prob_i} \sum_{i=1}^n prob_i T +  L \cdot n \cdot t_{\text{gate}}
\end{align}
\end{theorem}

\begin{proof}
Using \eqref{eq:convergence_fedAvg},  \eqref{eqn:encoding_time}  and \eqref{eqn:blockchain_time},
we can prove that the total time convergence is satisfied as in Theorem \ref{theorem:main_convergence}.
\end{proof}

\subsection{Metaverse Factors}

With important insights form \cite{sicknessReductionTechnology2020}, we can express the meta immersion experience $E_{meta}$ for any user $k$ as, 

\begin{align}
    E_{meta}^k = D_{rate}^k(1 - Ulink_{errorRate}^k) * VR_e^k
\end{align}
where, $D_{rate}$ is the download data link that impacts lossless virtual experience, $Ulink_{errorRate}$ is uplink tracking bit error rate and $VR_e$ is a virtual experience that is subjective to the user.

For quantifying virtual experience, we can say,
\begin{equation}
    VR_e \propto \{activity, onlineTime, ...\}
    \label{eqn:vr_experience}
\end{equation}

From Equation \ref{theorem:main_convergence} and \ref{eqn:vr_experience},
\begin{equation}
    VR_e \propto E[total\_time]
    \label{eqn:vr_experience_prop}
\end{equation}

From our experimental results \ref{fig:performance}, QFL performs faster than CFL.
Thus, this directly implies that with QFL we will have better $VR_e$ than CFL and, in general, better service indicators and technical indicators.

\begin{theorem}
    With BQFL, the experience of Metaverse $VR_e$ is always greater or more satisfactory than CFL.
\end{theorem}
\begin{proof}
     Different technical and service indicators such as $D_{rate}, Uplink_{errorRate}$, etc. are impacted proportionally to $VR_e$, $E[total\_time]$. Using \eqref{eqn:vr_experience_prop} and Theorem \ref{theorem:main_convergence}, we can prove \textit{Theorem 2}. We will discuss later how our experimental analysis supports the conclusion of \textit{Theorem 2}.
\end{proof}
\subsection{Metaverse Ecosystem}
Different aspects of Metaverse include user behavior prediction, content recommendation, object
recognition, training data, etc.
The metaverse can be considered as a network
of interconnected nodes, where each node is a user,
object, or virtual space, and edges are the relationships or
connections between them. 
Graph theory can be used for mathematical
analysis and modeling of such networks.
In the metaverse, various events,
interactions, and behaviors occur probabilistically. 
Thus, probability
theory can be used to model the likelihood of events happening
like the probability of encountering a particular object or
meeting a specific user. 
With statistical analysis,  
understanding patterns, trends, and distributions within the
metaverse can be achieved.
Also, user behavior can be analyzed within the metaverse in a probabilistic manner.
Finally, machine learning algorithms could
be used to predict and simulate user actions based on historical
data, contextual information, and user preferences.

To this end, we have considered three key aspects of metaverse ecosystems that can be orchestrated with BQFL. They are:
\subsubsection{PQ security}
\cite{duanMetaverseSocialGood2021b} For a fair and transparent ecosystem, security is crucial.
This demonstrates an impeding need for the post-quantum secure BQFL.

\subsubsection{Autonomous Governance}
Autonomous Governance is the key to the success of the whole system.
This prevents the system from being controlled by a certain group of people.

\subsubsection{AI-Driven Metaverse Observer}
Duan \textit{et. al} \cite{duanMetaverseSocialGood2021b} presented the idea of an AI-Driven Metaverse Observer, who can track real-time operation data from the Metaverse and analyze it. 
This observer can make a recommendation on ongoing events to users.
For this approval rating system can be implemented.
This would provide global information that can assist users to capture timely events in a better way.

\section{Experiments and Results}
To study the integration of blockchain in QFL, we have inherited implementation approaches in 
\cite{pokhrelFederatedLearningBlockchain2020},  
\cite{chenRobustBlockchainedFederated2021} for BCFL and \cite{zhaoExactDecompositionQuantum2022a} for QFL.
The experiments were run in Google Colab Pro as well as on a local computer.

\subsection{Preprocessing}
Image data are preprocessed for training and testing purposes as follows.
The first pixel values of the input images are scaled in the range of [0,1].
After that, encoding is applied to the images depending on the type of encoding.
With the "vanilla" encoding, the mean is set to 0.
With "mean" encoding, the mean of training images is subtracted from all images.
While with the "half" encoding approach, the images are shifted by 0.5.

Another important step in preprocessing is resizing the image to the size of
$[int(2^{n/2}), int(2^{n/2})]$, where $n$ is the number of qubits used in the quantum circuit.
Eventually, the resized images are flattened to a 1D array of size $2^n$.
Finally, the pixel values are normalized by dividing each image by the square root of the sum of its squared pixel values so that we have a unit length.
The labels for the input images are hot encoded to match the output format of the model.

\subsection{Dataset Preparation}
Quantum computers cannot directly process the classical representations of datasets. The data preparation consists of the following steps:
\begin{enumerate}
    \item Data Loading: Libraries like TensorFlow can be used for the loading and splitting of data sets into a training and test set.
    With the initial processing of normalization and encoding, the downscaling of the images needs to be performed afterward.
    \item Downscaling images: An image size of 28 × 28 is too large for existing quantum computers. 
    Thus, they need to be resized to a size of 16 × 16 (for an 8-qubit quantum circuit) or \num{4} X \num{4}.
    \item Data Encoding is an essential step in QML.
    We perform encoding of classical data into states of qubits.
\end{enumerate}

We use the MNIST dataset for experimental purposes.
MNIST dataset consists of 70,000, 28 × 28 images of handwritten digits.
The digits consist of 10 classes (0 to 9).
For this work, we have performed data sharding similar to \cite{zhaoExactDecompositionQuantum2022a}.
In doing so, we remove the samples with labels equal to '8' and '9' from both the training and the testing sets.
We follow a cycle-m structure, where each client has access to the data set to only
$m$ classes at a time.
We also consider $n = 9$ clients as a whole with $7$ assigned as workers and $2$ miners.
Both quantum FedAvg and quantum FedInference \cite{zhaoExactDecompositionQuantum2022a} are experimented with.
For classical learning, normal FedAveraging is used as in \cite{chenRobustBlockchainedFederated2021}.
The batch size is 128 and the learning rate is 0.01.
In terms of evaluation, top 1 accuracy and loss are used as performance metrics.

\begin{figure}
    \centering
    \begin{subfigure}[]{\columnwidth}
    \centering
    \includegraphics[width=0.5\columnwidth]{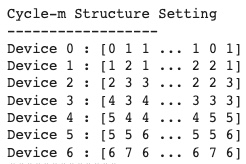}
    \caption{Cycle-m Sharding}
    \label{fig:cycle}
    \end{subfigure}
    \caption{Dataset Sharding}
    \label{fig:data_sharding}
\end{figure}

Figures 
\ref{fig:bfql-inf}, 
\ref{fig:bqfl-avg}, and \ref{fig:bcfl} display the individual test 
accuracy plots for 
BQFL-inf, 
BQFL-avg, and BCFL-avg, respectively.
BQFL-inf performs exceptionally well with non-IID data, as evident from Figure \ref{fig:bfql-inf}.
However, BQFL-avg struggles more with non-IID data, as depicted in Figure \ref{fig:bqfl-avg}, 
with highly fluctuating accuracy between the Top 1 and the lowest accuracy, as shown in Figure \ref{fig:obervations_avg}.
On the other hand, BCFL-avg performs well with the non-IID dataset, as illustrated in Figure \ref{fig:bcfl}.

\begin{figure}
    \centering
    \begin{subfigure}[]{0.45\columnwidth}
    \centering
    \includegraphics[width=\columnwidth]{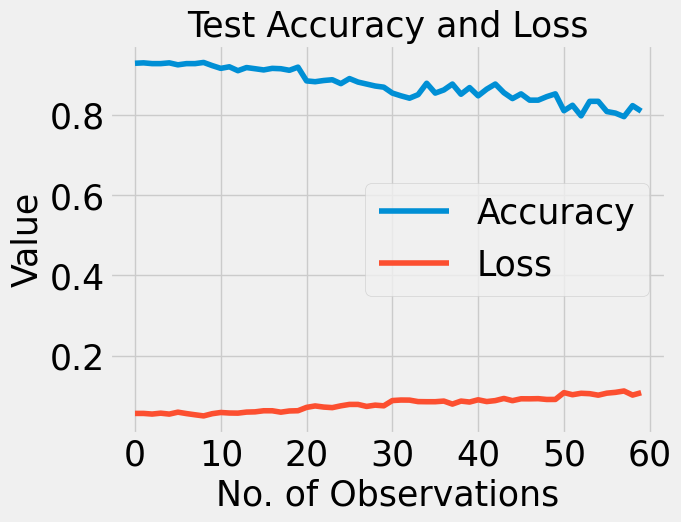}
    \caption{Accuracy observations}
    \label{fig:obervations_inf}
    \end{subfigure}
    \centering
   \begin{subfigure}[]{0.45\columnwidth}
    \centering
    \includegraphics[width=\columnwidth]{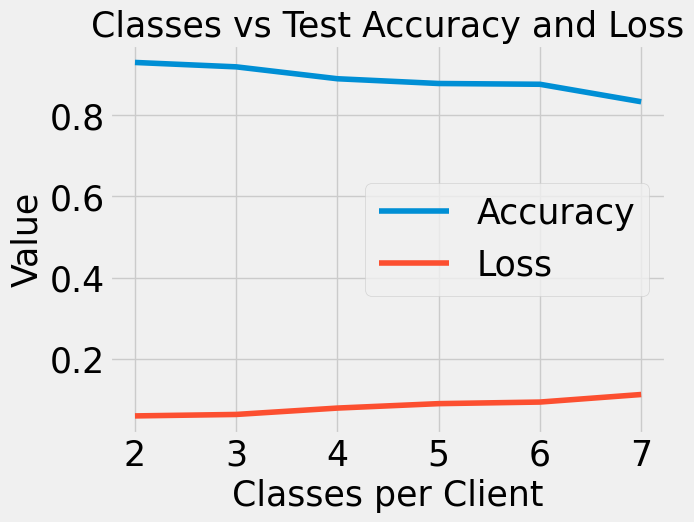}
    \caption{Against Classes}
    \label{fig:against_classes_inf}
    \end{subfigure}
    \caption{BQFL-inference Test Accuracy}
    \label{fig:bfql-inf}
\end{figure}

\begin{figure}
    \centering
     \begin{subfigure}[]{0.45\columnwidth}
    \centering
    \includegraphics[width=\columnwidth]{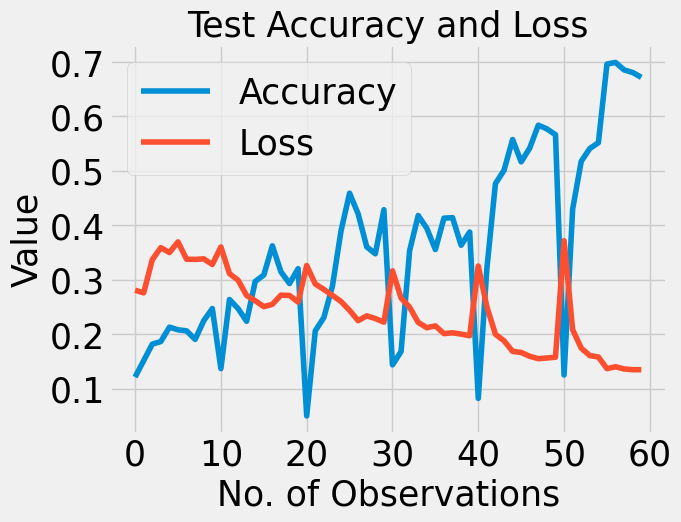}
    \caption{Accuracy observations}
    \label{fig:obervations_avg}
    \end{subfigure}
    \centering
    \begin{subfigure}[]{0.45\columnwidth}
    \centering
    \includegraphics[width=\columnwidth]{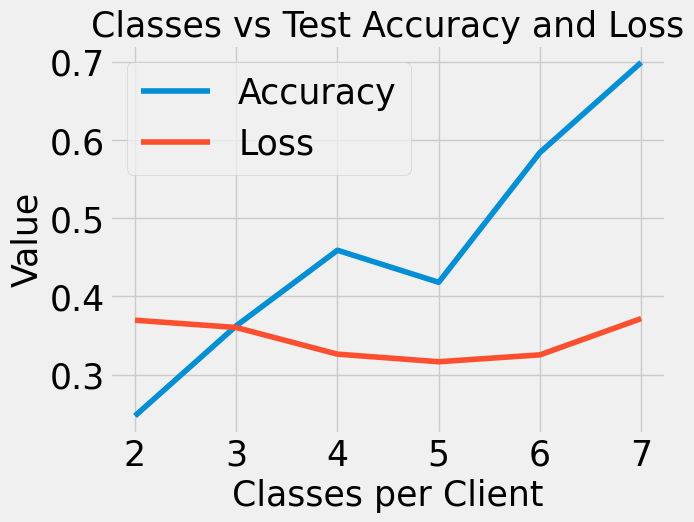}
    \caption{Against Classes}
    \label{fig:against_classes_avg}
    \end{subfigure}
    \caption{BQFL-avg Test Accuracy}
    \label{fig:bqfl-avg}
\end{figure}
\begin{figure}
    \centering
    \begin{subfigure}[]{0.45\columnwidth}
    \centering
    \includegraphics[width=\columnwidth]{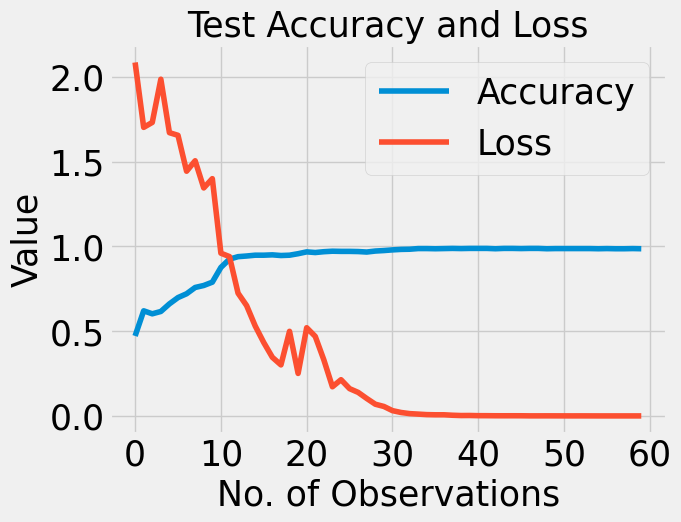}
    \caption{Accuracy observations}
    \label{fig:obervations_bfl}
    \end{subfigure}
    \centering
   \begin{subfigure}[]{0.45\columnwidth}
    \centering
    \includegraphics[width=\columnwidth]{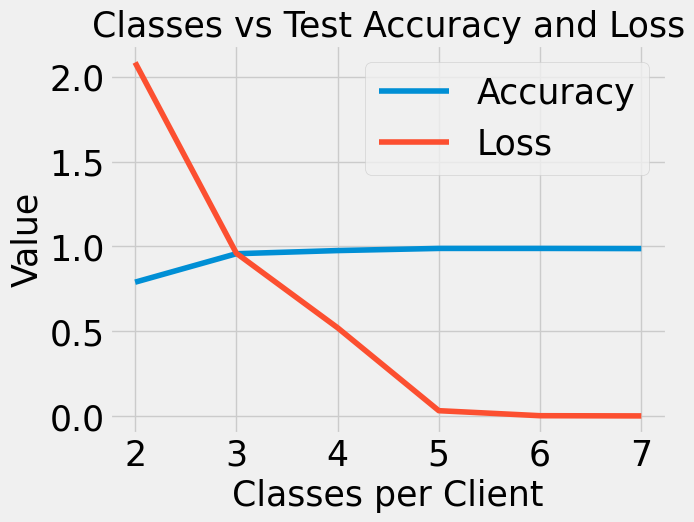}
    \caption{Against Classes}
    \label{fig:against_classes_bfl}
    \end{subfigure}
    \caption{BCFL-avg Test Accuracy}
    \label{fig:bcfl}
\end{figure}

\subsection{Test Performance}

Figure \ref{fig:test_performance} shows the test accuracy plots for BQFL-avg, 
BQFL-inf, 
and BCFL-avg. Among these, BCFL-avg outperforms 
both BQFL-inf and 
BQFL-avg in terms of test accuracy. 
As the degree of non-IID decreases, the test accuracy of BCFL-avg increases. 
However, for BQFL-inf, the test accuracy decreases slightly with an increase in the degree of non-IID. 
On the other hand, BQFL-avg suffers greatly with a higher degree of non-IID, especially when training workers with each having only two classes, resulting in a low final test accuracy.
\begin{figure}
    \centering
    \begin{subfigure}[]{0.45\columnwidth}
    \centering
    \includegraphics[width=\columnwidth]{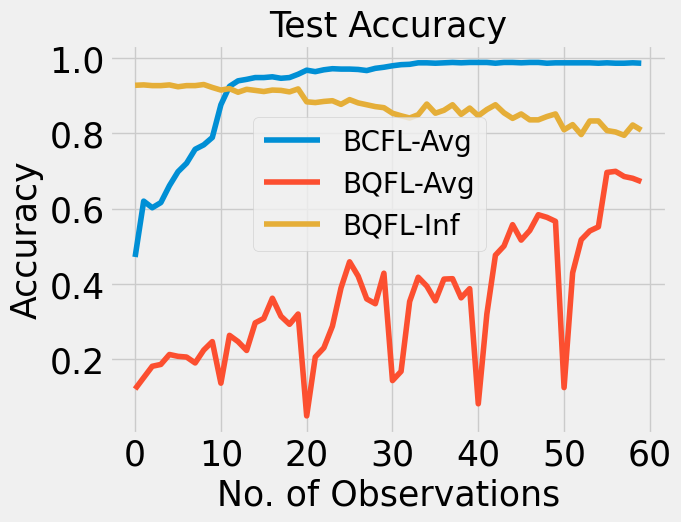}
    \caption{Accuracy observations}
    \label{fig:observations_all}
    \end{subfigure}
    \centering
   \begin{subfigure}[]{0.45\columnwidth}
    \centering
    \includegraphics[width=\columnwidth]{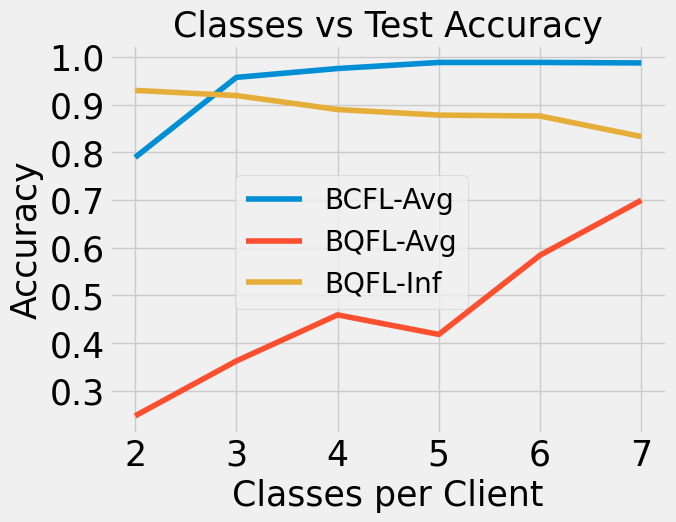}
    \caption{Against Classes}
    \label{fig:against_classes_all}   
    \end{subfigure}
    \caption{Test Accuracy}
    \label{fig:test_performance}
\end{figure}

\subsection{Training Performance}
In addition to evaluating the test accuracy, we also analyzed the training performance of the
different FL frameworks. 
As shown in Figure \ref{fig:training_performance}, both BQFL-avg 
and BQFL-inf 
converge faster than BCFL-avg.
BQFL-avg and BQFL-inf share similar training performance, indicating that the additional
communication overhead incurred by BQFL-inf does not result in significant performance degradation.
However, as shown in Figure \ref{fig:test_performance}, BQFL-inf performance is declining when the degree of non-IID increases.
This is the opposite of BQFL-avg in terms of
test accuracy.
In contrast, as shown in Figure \ref{fig:training_performance},  BCFL-avg has a slower convergence rate compared to the BQFL frameworks, 
especially when the degree of non-IID is high in terms of training. 
Overall, these results suggest that BQFL-avg 
and BQFL-inf 
can achieve better training performance
than BCFL-avg, while BQFL-inf may not be as robust as BQFL-avg in handling non-IID data.

\begin{figure}
    \centering 
    \begin{subfigure}[]{0.45\columnwidth}
    \centering
    \includegraphics[width=\columnwidth]{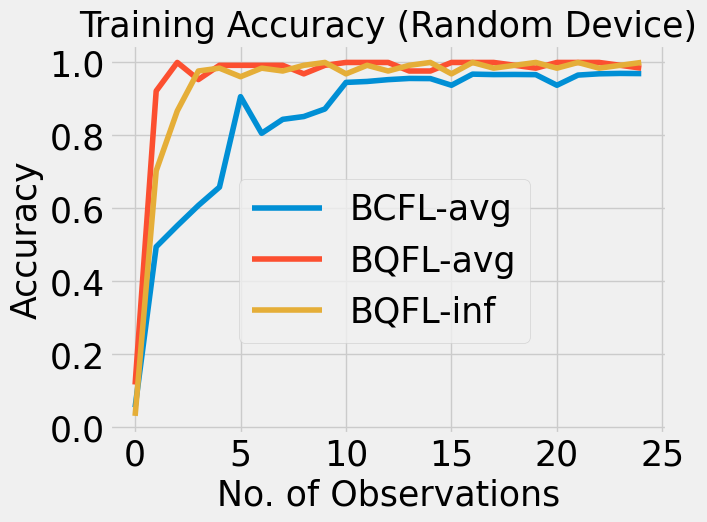}
    \caption{Training Accuracy}
    \label{fig:against_classes}
    \end{subfigure}
    \centering
   \begin{subfigure}[]{0.45\columnwidth}
    \centering
    \includegraphics[width=\columnwidth]{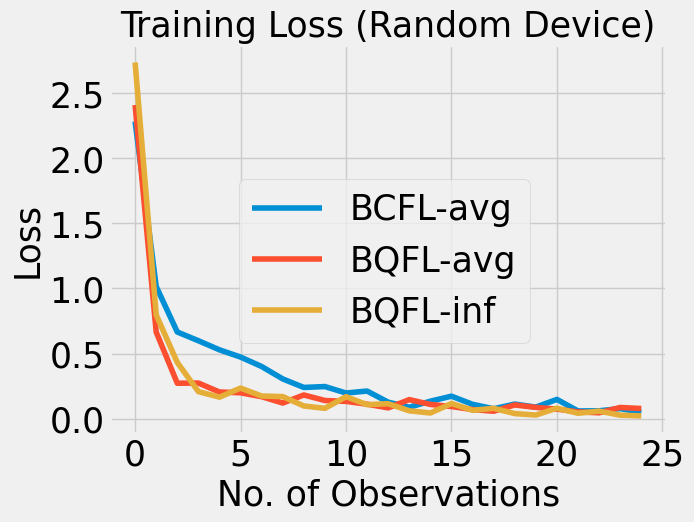}
    \caption{Training Loss}
    \label{fig:obervations}
    \end{subfigure}
    \caption{Training Performance}
    \label{fig:training_performance}
\end{figure}

\subsection{Impact of Degree of Non-IID}
The impact of the degree of non-IID on the test accuracy of BQFL-Avg, 
BQFL-inf and BCFL-
avg is shown in Figure \ref{fig:against_classes_all}. The results reveal that the degree of non-IID 
has varying effects on the performance of the different federated learning algorithms.
First, we observe that  BQFL-Avg is the most impacted by a higher degree of non-IID, as evidenced by
its decreasing test accuracy with increasing non-IID. 
This indicates that BQFL-Avg may not be the 
best choice for federated learning scenarios with highly non-IID data distributions.
In contrast, BQFL-inf is not that significantly impacted by the degree of non-IID. 
This suggests that 
BQFL-inf may be a suitable algorithm for federated learning with non-IID data distributions.
BQFL-avg, on the other hand, shows a more obvious decrease in test accuracy with increasing non-
IID, particularly at higher levels. 
This indicates that BQFL-avg may not perform well in 
federated learning scenarios with highly non-IID data distributions.
Finally, we observe that BCFL-avg performs consistently well across all levels of non-IID. 
In fact,
BCFL-avg outperforms BQFL-avg at all levels of non-IID, indicating that BCFL-avg may be a more robust
algorithm for federated learning with non-IID data distributions.

\subsection{Quantifying Stake Accumulation}
The plot in Figure \ref{fig:stake} indicates that there is a similar trend in stake accumulation for 
all three cases of BQFL-avg, 
BQFL-inf, 
and BCFL-avg. 
However, it is worth noting that there are some variations in the initial stages of stake 
accumulation, especially for BCFL-avg. 
As shown in the plot, stake accumulation starts at a relatively lower level for BCFL-avg, 
but it catches up with the other two methods as the stake accumulation progresses.
It is important to consider stake accumulation as it directly affects the selection of 
representatives for the consensus protocol.
In scenarios where a selection mechanism is implemented, 
higher stake accumulation indicates a higher 
probability of a node being selected as a representative, which in turn, 
increases its influence in the consensus process. 
Therefore, having a steady and predictable stake accumulation rate is crucial for the stability and 
security of the consensus protocol.
However, its actual implementation is limited in this work.

\begin{figure}[!h]
    \centering
     \begin{subfigure}[]{0.45\columnwidth}
    \centering
    \includegraphics[width=\columnwidth]{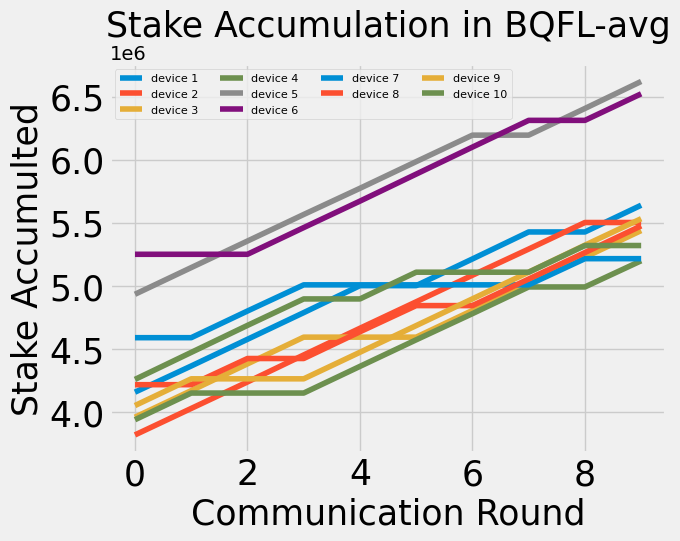}
    \caption{BQFL-avg}
    \label{fig:stake_avg}
    \end{subfigure}
    \centering
    \begin{subfigure}[]{0.45\columnwidth}
    \centering
    \includegraphics[width=\columnwidth]{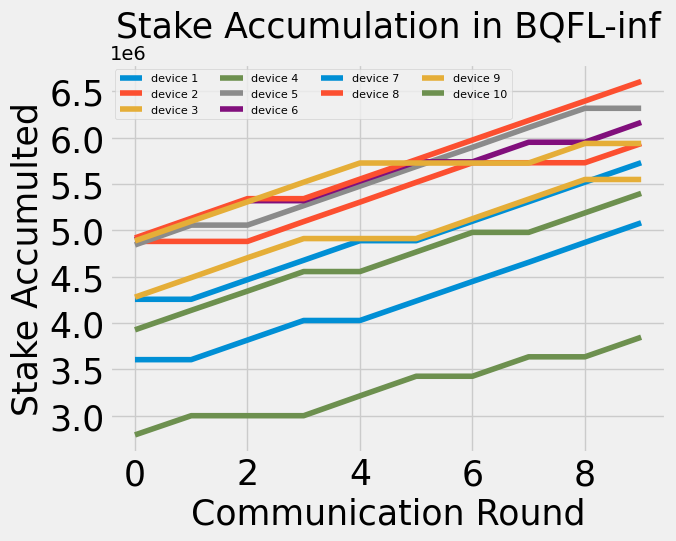}
    \caption{BQFL-inf}
    \label{fig:stake_inf}
    \end{subfigure}
    
     \centering
    \begin{subfigure}[]{0.45\columnwidth}
    \centering
    \includegraphics[width=\columnwidth]{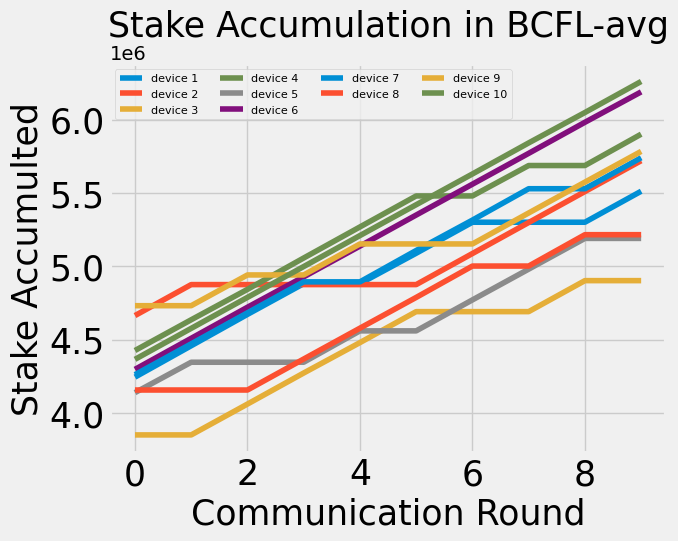}
    \caption{BCFL-avg}
    \label{fig:stake_bcfl}
    \end{subfigure}
    \caption{Stake Accumulation}
    \label{fig:stake}
\end{figure}


\subsection{Delay Performance}
In terms of communication time as shown in Figure \ref{fig:comm_time}, 
BQFL-inf takes the longest time compared to the other algorithms.
BQFL-avg is faster than BCFL-avg in this aspect, which indicates that BQFL-avg can achieve faster
convergence compared to BCFL-avg. However, it's important to note that there is a slight difference
in the way test accuracy and test loss are computed between BQFL and BCFL.
Regarding block generation time as shown in Figure \ref{fig:block_generation}, BCFL-avg takes the longest time of all algorithms. 
BQFL-avg and BQFL-inf 
have similar performance in terms of block generation time. It is worth mentioning that block 
generation time can have a significant impact on the overall performance of the federated learning
algorithm, especially in scenarios where the network has limited resources.
Therefore, a trade-off must be made between communication time and block generation time to achieve optimal performance for a given system.

\begin{figure}
    \centering
    \begin{subfigure}[]{0.45\columnwidth}
    \centering
    \includegraphics[width=\columnwidth]{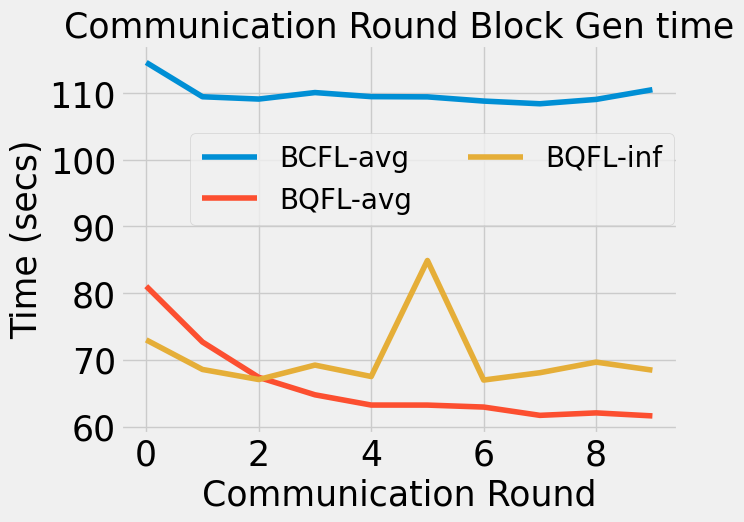}
    \caption{Block Generation Time}
    \label{fig:block_generation}
    \end{subfigure}
    \centering
   \begin{subfigure}[]{0.45\columnwidth}
    \centering
    \includegraphics[width=\columnwidth]{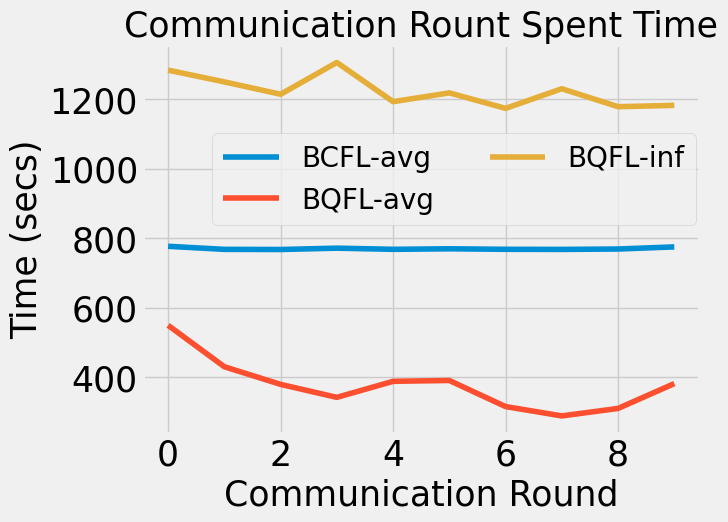}
    \caption{Communication Time}
    \label{fig:comm_time}
    \end{subfigure}
    \caption{Performance}
    \label{fig:performance}
\end{figure}
\subsection{Accounting Metaverse measures}
For a fully immersive experience in the metaverse, a few metrics are known as the service and technical indicators of user experience and feelings in the metaverse
For example, for a lossless visual experience, a download data link is required to be high enough, i.e., 20–40 Mbps \cite{sicknessReductionTechnology2020}, which impacts the resolution, frame rate, motion blur, etc., that determines the feeling of presence and are service indicators. 
In this instance, BQFL can assist in achieving its goal because of its high-performance training and testing, which can be used to create digital twins or any other tasks faster.
From Figures \ref{fig:performance} and \ref{fig:comparison_average_min_max}, it is clear that BQFL-avg performs better than BCFL-avg implying that it is better suited for applications such as metaverse and can fulfill today's increasing data and computational needs.

\begin{figure}
    \centering
    \begin{subfigure}[]{0.45\columnwidth}
    \centering
    \includegraphics[width=\columnwidth]{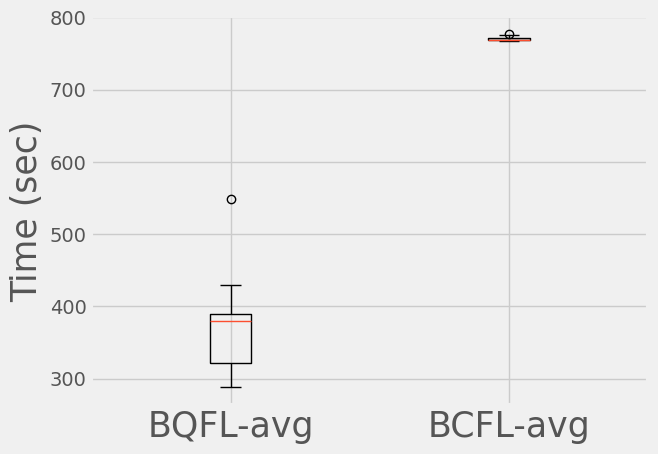}
    \caption{Max Min Time}
    \label{fig:block_generation_minmax}
    \end{subfigure}
    \centering
   \begin{subfigure}[]{0.45\columnwidth}
    \centering
    \includegraphics[width=\columnwidth]{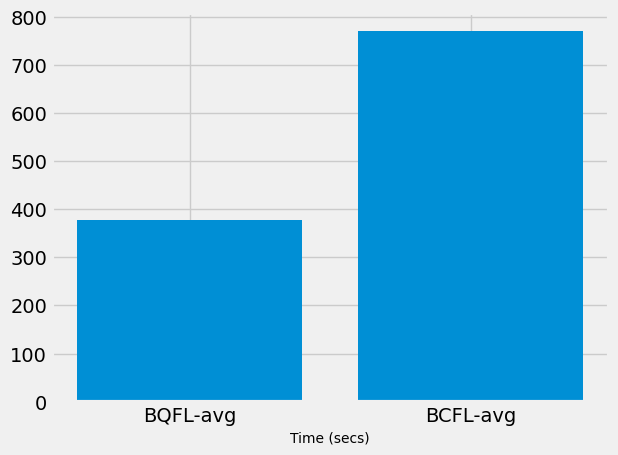}
    \caption{Average Time}
    \label{fig:comm_time_average}
    \end{subfigure}
    \caption{Communication Delay Comparison}
    \label{fig:comparison_average_min_max}
\end{figure}

\section{Concluding Remarks}
In this work, we have developed a rigorous analysis and design of a BQFL framework, considering
its practicality and implementation in Metaverse.
Extensive theoretical and experimental analysis is done to design and understand the behavior of integration between QFL with the blockchain.
 We have developed new insights and explained significant results with new findings. Our experimental results demonstrated the practicality of BQFL.
 However, extensive further research is required to fully understand the practicality of BQFL and its application for the metaverse application.
 which requires further investigation.


\begin{thebibliography}{10}
\providecommand{\url}[1]{#1}
\csname url@samestyle\endcsname
\providecommand{\newblock}{\relax}
\providecommand{\bibinfo}[2]{#2}
\providecommand{\BIBentrySTDinterwordspacing}{\spaceskip=0pt\relax}
\providecommand{\BIBentryALTinterwordstretchfactor}{4}
\providecommand{\BIBentryALTinterwordspacing}{\spaceskip=\fontdimen2\font plus
\BIBentryALTinterwordstretchfactor\fontdimen3\font minus
  \fontdimen4\font\relax}
\providecommand{\BIBforeignlanguage}[2]{{%
\expandafter\ifx\csname l@#1\endcsname\relax
\typeout{** WARNING: IEEEtran.bst: No hyphenation pattern has been}%
\typeout{** loaded for the language `#1'. Using the pattern for}%
\typeout{** the default language instead.}%
\else
\language=\csname l@#1\endcsname
\fi
#2}}
\providecommand{\BIBdecl}{\relax}
\BIBdecl

\bibitem{biamonte2017quantum}
J.~Biamonte, P.~Wittek, N.~Pancotti, P.~Rebentrost, N.~Wiebe, and S.~Lloyd,
  ``Quantum machine learning,'' \emph{Nature}, vol. 549, no. 7671, pp.
  195--202, 2017.

\bibitem{kwakQuantumDistributedDeep2022}
\BIBentryALTinterwordspacing
Y.~Kwak \emph{et~al.}, ``Quantum distributed deep learning architectures:
  {{Models}}, discussions, and applications.'' [Online]. Available:
  \url{https://www.sciencedirect.com/science/article/pii/S2405959522001138}
\BIBentrySTDinterwordspacing

\bibitem{abbasPowerQuantumNeural2021}
\BIBentryALTinterwordspacing
A.~Abbas \emph{et~al.}, ``The power of quantum neural networks,'' vol.~1,
  no.~6, pp. 403--409, comment: 25 pages, 10 figures. [Online]. Available:
  \url{http://arxiv.org/abs/2011.00027}
\BIBentrySTDinterwordspacing

\bibitem{pokhrelFederatedLearningBlockchain2020}
S.~R. Pokhrel and J.~Choi, ``Federated {{Learning With Blockchain}} for
  {{Autonomous Vehicles}}: {{Analysis}} and {{Design Challenges}},'' vol.~68,
  no.~8, pp. 4734--4746.

\bibitem{pokhrelBlockchainBringsTrust2021}
S.~R. Pokhrel, ``Blockchain {{Brings Trust}} to {{Collaborative Drones}} and
  {{LEO Satellites}}: {{An Intelligent Decentralized Learning}} in the
  {{Space}},'' vol.~21, no.~22, pp. 25\,331--25\,339.

\bibitem{larasatiQuantumFederatedLearning2022}
H.~T. Larasati, M.~Firdaus, and H.~Kim, ``Quantum {{Federated Learning}}:
  {{Remarks}} and {{Challenges}},'' in \emph{2022 {{IEEE}} 9th {{International
  Conference}} on {{Cyber Security}} and {{Cloud Computing}} ({{CSCloud}})/2022
  {{IEEE}} 8th {{International Conference}} on {{Edge Computing}} and
  {{Scalable Cloud}} ({{EdgeCom}})}, pp. 1--5.

\bibitem{kaewpuangAdaptiveResourceAllocation2022a}
\BIBentryALTinterwordspacing
R.~Kaewpuang, M.~Xu, D.~Niyato, H.~Yu, Z.~Xiong, and X.~S. Shen, ``Adaptive
  resource allocation in quantum key distribution ({QKD}) for federated
  learning.'' [Online]. Available: \url{http://arxiv.org/abs/2208.11270}
\BIBentrySTDinterwordspacing

\bibitem{yangDecentralizingFeatureExtraction2021a}
C.-H.~H. Yang, J.~Qi, S.~Y.-C. Chen, P.-Y. Chen, S.~M. Siniscalchi, X.~Ma, and
  C.-H. Lee, ``Decentralizing {{Feature Extraction}} with {{Quantum
  Convolutional Neural Network}} for {{Automatic Speech Recognition}},'' in
  \emph{{{ICASSP}} 2021 - 2021 {{IEEE International Conference}} on
  {{Acoustics}}, {{Speech}} and {{Signal Processing}} ({{ICASSP}})}, pp.
  6523--6527.

\bibitem{yunSlimmableQuantumFederated}
W.~J. Yun \emph{et~al.}, ``Slimmable {{Quantum Federated Learning}}.''

\bibitem{xiaQuantumFedFederatedLearning2021a}
Q.~Xia and Q.~Li, ``{{QuantumFed}}: {{A Federated Learning Framework}} for
  {{Collaborative Quantum Training}},'' in \emph{2021 {{IEEE Global
  Communications Conference}} ({{GLOBECOM}})}, pp. 1--6.

\bibitem{xiaDefendingByzantineAttacks2021}
Q.~Xia, Z.~Tao, and Q.~Li, ``Defending {{Against Byzantine Attacks}} in
  {{Quantum Federated Learning}},'' in \emph{2021 17th {{International
  Conference}} on {{Mobility}}, {{Sensing}} and {{Networking}} ({{MSN}})}, pp.
  145--152.

\bibitem{zhangFederatedLearningQuantum2022}
\BIBentryALTinterwordspacing
Y.~Zhang, C.~Zhang, C.~Zhang, L.~Fan, B.~Zeng, and Q.~Yang, ``Federated
  {{Learning}} with {{Quantum Secure Aggregation}}.'' [Online]. Available:
  \url{http://arxiv.org/abs/2207.07444}
\BIBentrySTDinterwordspacing

\bibitem{gurungSECURECOMMUNICATIONMODEL2023}
\BIBentryALTinterwordspacing
D.~Gurung, S.~Pokhrel, and G.~Li, ``{{SECURE COMMUNICATION MODEL FOR QUANTUM
  FEDERATED LEARNING}}: {{A PROOF OF CONCEPT}}.'' [Online]. Available:
  \url{https://openreview.net/forum?id=xZGPLvRpf4N}
\BIBentrySTDinterwordspacing

\bibitem{qiFederatedQuantumNatural2022}
\BIBentryALTinterwordspacing
J.~Qi, ``Federated {{Quantum Natural Gradient Descent}} for {{Quantum Federated
  Learning}},'' comment: Published parts of book in Federated Learning.
  [Online]. Available: \url{http://arxiv.org/abs/2209.00564}
\BIBentrySTDinterwordspacing

\bibitem{yamanyOQFLOptimizedQuantumBased2021}
W.~Yamany, N.~Moustafa, and B.~Turnbull, ``{{OQFL}}: {{An Optimized
  Quantum-Based Federated Learning Framework}} for {{Defending Against
  Adversarial Attacks}} in {{Intelligent Transportation Systems}},'' pp. 1--11.

\bibitem{huangQuantumFederatedLearning2022}
R.~Huang, X.~Tan, and Q.~Xu, ``Quantum {{Federated Learning With Decentralized
  Data}},'' vol.~28, pp. 1--10.

\bibitem{chehimiQuantumFederatedLearning2022}
M.~Chehimi and W.~Saad, ``Quantum {{Federated Learning}} with {{Quantum
  Data}},'' in \emph{{{ICASSP}} 2022 - 2022 {{IEEE International Conference}}
  on {{Acoustics}}, {{Speech}} and {{Signal Processing}} ({{ICASSP}})}, pp.
  8617--8621.

\bibitem{pokhrelFederatedLearningMeets2020}
\BIBentryALTinterwordspacing
S.~R. Pokhrel, ``Federated learning meets blockchain at {{6G}} edge: A
  drone-assisted networking for disaster response,'' in \emph{Proceedings of
  the 2nd {{ACM MobiCom Workshop}} on {{Drone Assisted Wireless
  Communications}} for {{5G}} and {{Beyond}}}, ser. {{DroneCom}} '20.\hskip 1em
  plus 0.5em minus 0.4em\relax {Association for Computing Machinery}, pp.
  49--54. [Online]. Available: \url{https://doi.org/10.1145/3414045.3415949}
\BIBentrySTDinterwordspacing

\bibitem{chenRobustBlockchainedFederated2021}
\BIBentryALTinterwordspacing
H.~Chen \emph{et~al.}, ``Robust {{Blockchained Federated Learning}} with
  {{Model Validation}} and {{Proof-of-Stake Inspired Consensus}},'' comment: 8
  pages, 7 figures, AAAI 2021 Workshop - Towards Robust, Secure and Efficient
  Machine Learning. [Online]. Available: \url{http://arxiv.org/abs/2101.03300}
\BIBentrySTDinterwordspacing

\bibitem{pokhrelDecentralizedFederatedLearning2020}
S.~R. Pokhrel and J.~Choi, ``A {{Decentralized Federated Learning Approach}}
  for {{Connected Autonomous Vehicles}},'' in \emph{2020 {{IEEE Wireless
  Communications}} and {{Networking Conference Workshops}} ({{WCNCW}})}, pp.
  1--6.

\bibitem{zhaoExactDecompositionQuantum2022a}
\BIBentryALTinterwordspacing
H.~Zhao, ``Exact {{Decomposition}} of {{Quantum Channels}} for {{Non-IID
  Quantum Federated Learning}},'' comment: 6 pages excluding appendices and
  references. Code available at https://github.com/JasonZHM/quantum-fed-infer.
  [Online]. Available: \url{http://arxiv.org/abs/2209.00768}
\BIBentrySTDinterwordspacing

\bibitem{duanMetaverseSocialGood2021b}
\BIBentryALTinterwordspacing
H.~Duan, J.~Li, S.~Fan, Z.~Lin, X.~Wu, and W.~Cai, ``Metaverse for {{Social
  Good}}: {{A University Campus Prototype}},'' in \emph{Proceedings of the 29th
  {{ACM International Conference}} on {{Multimedia}}}, pp. 153--161. [Online].
  Available: \url{http://arxiv.org/abs/2108.08985}
\BIBentrySTDinterwordspacing

\bibitem{yangFusingBlockchainAI2022a}
Q.~Yang, Y.~Zhao, H.~Huang, Z.~Xiong, J.~Kang, and Z.~Zheng, ``Fusing
  {{Blockchain}} and {{AI With Metaverse}}: {{A Survey}},'' vol.~3, pp.
  122--136.

\bibitem{metaverseP2PGaming}
\BIBentryALTinterwordspacing
P.~Bhattacharya, A.~Verma, V.~K. Prasad, S.~Tanwar, B.~Bhushan, B.~C. Florea,
  D.~D. Taralunga, F.~Alqahtani, and A.~Tolba, ``Game-o-meta: Trusted federated
  learning scheme for p2p gaming metaverse beyond 5g networks,''
  \emph{Sensors}, vol.~23, no.~9, 2023. [Online]. Available:
  \url{https://www.mdpi.com/1424-8220/23/9/4201}
\BIBentrySTDinterwordspacing

\bibitem{chang6GEnabledEdgeAI2022}
L.~Chang, Z.~Zhang, P.~Li, S.~Xi, W.~Guo, Y.~Shen, Z.~Xiong, J.~Kang,
  D.~Niyato, X.~Qiao, and Y.~Wu, ``{{6G-Enabled Edge AI}} for {{Metaverse}}:
  {{Challenges}}, {{Methods}}, and {{Future Research Directions}},'' vol.~7,
  no.~2, pp. 107--121.

\bibitem{zengHFedMSHeterogeneousFederated2022}
\BIBentryALTinterwordspacing
S.~Zeng, Z.~Li, H.~Yu, Z.~Zhang, L.~Luo, B.~Li, and D.~Niyato. {{HFedMS}}:
  {{Heterogeneous Federated Learning}} with {{Memorable Data Semantics}} in
  {{Industrial Metaverse}}. [Online]. Available:
  \url{http://arxiv.org/abs/2211.03300}
\BIBentrySTDinterwordspacing

\bibitem{preskillQuantumComputingNISQ2018}
\BIBentryALTinterwordspacing
J.~Preskill, ``Quantum {{Computing}} in the {{NISQ}} era and beyond,'' vol.~2,
  p.~79, comment: 20 pages. Based on a Keynote Address at Quantum Computing for
  Business, 5 December 2017. (v3) Formatted for publication in Quantum, minor
  revisions. [Online]. Available: \url{http://arxiv.org/abs/1801.00862}
\BIBentrySTDinterwordspacing

\bibitem{baevskiData2vecGeneralFramework}
A.~Baevski, W.-N. Hsu, Q.~Xu, A.~Babu, J.~Gu, and M.~Auli, ``Data2vec: {{A
  General Framework}} for {{Self-supervised Learning}} in {{Speech}},
  {{Vision}} and {{Language}}.''

\bibitem{QuantumDataBaidu}
``Encoding classical data into quantum states,''
  \url{https://qml.baidu.com/tutorials/machine-learning/encoding-classical-data-into-quantum-states.html},
  (Accessed on 11/15/2022).

\bibitem{zhangTensorCircuitQuantumSoftware2022a}
\BIBentryALTinterwordspacing
S.-X. Zhang \emph{et~al.}, ``{{TensorCircuit}}: A {{Quantum Software
  Framework}} for the {{NISQ Era}},'' comment: Whitepaper for TensorCircuit, 43
  pages, 11 figures, 8 tables. [Online]. Available:
  \url{http://arxiv.org/abs/2205.10091}
\BIBentrySTDinterwordspacing

\bibitem{yunSlimmableQuantumFederated2022}
\BIBentryALTinterwordspacing
W.~J. Yun \emph{et~al.}, ``Slimmable {{Quantum Federated Learning}}.''
  [Online]. Available: \url{https://arxiv.org/abs/2207.10221v1}
\BIBentrySTDinterwordspacing

\bibitem{arthurHybridQuantumClassicalNeural2022}
\BIBentryALTinterwordspacing
D.~Arthur and P.~Date, ``A {{Hybrid Quantum-Classical Neural Network
  Architecture}} for {{Binary Classification}},'' comment: Added reference to
  section I. Fixed error in methods (Section III.C). [Online]. Available:
  \url{http://arxiv.org/abs/2201.01820}
\BIBentrySTDinterwordspacing

\bibitem{cerezoVariationalQuantumAlgorithms2021}
\BIBentryALTinterwordspacing
M.~Cerezo, A.~Arrasmith, R.~Babbush, S.~C. Benjamin, S.~Endo, K.~Fujii, J.~R.
  McClean, K.~Mitarai, X.~Yuan, L.~Cincio, and P.~J. Coles, ``Variational
  {{Quantum Algorithms}},'' vol.~3, no.~9, pp. 625--644, comment: Review
  Article. 33 pages, 7 figures. Updated to published version. [Online].
  Available: \url{http://arxiv.org/abs/2012.09265}
\BIBentrySTDinterwordspacing

\bibitem{jerbiQuantumMachineLearning2023}
\BIBentryALTinterwordspacing
S.~Jerbi, L.~J. Fiderer, H.~Poulsen~Nautrup, J.~M. Kübler, H.~J. Briegel, and
  V.~Dunjko, ``Quantum machine learning beyond kernel methods,'' vol.~14,
  no.~1, p. 517. [Online]. Available:
  \url{https://www.nature.com/articles/s41467-023-36159-y}
\BIBentrySTDinterwordspacing

\bibitem{gargAdvancesQuantumDeep2020}
\BIBentryALTinterwordspacing
S.~Garg and G.~Ramakrishnan. Advances in {{Quantum Deep Learning}}: {{An
  Overview}}. [Online]. Available: \url{http://arxiv.org/abs/2005.04316}
\BIBentrySTDinterwordspacing

\bibitem{liConvergenceFedAvgNonIID2020}
X.~Li, K.~Huang, W.~Yang, S.~Wang, and Z.~Zhang, ``On the convergence of fedavg
  on non-iid data,'' \emph{arXiv preprint arXiv:1907.02189}, 2019.

\bibitem{weigoldEncodingPatternsQuantum2021}
\BIBentryALTinterwordspacing
M.~Weigold, J.~Barzen, F.~Leymann, and M.~Salm, ``Encoding patterns for quantum
  algorithms,'' vol.~2, no.~4, pp. 141--152. [Online]. Available:
  \url{https://onlinelibrary.wiley.com/doi/abs/10.1049/qtc2.12032}
\BIBentrySTDinterwordspacing

\bibitem{sicknessReductionTechnology2020}
``Ieee standard for head-mounted display (hmd)-based virtual reality(vr)
  sickness reduction technology,'' \emph{IEEE Std 3079-2020}, pp. 1--74, 2021.

\end{thebibliography}



\end{document}